\documentclass[11pt]{article}

\usepackage[utf8]{inputenc}
\usepackage{amsfonts}
\usepackage{amssymb}
\usepackage{xcolor}
\usepackage{algpseudocode}
\usepackage{algorithm}

\usepackage{nicefrac}

\usepackage{fullpage,tikz,enumitem}

\usepackage[colorlinks=true, pdfstartview=FitV, 
linkcolor=black, citecolor=black, 
urlcolor=black, plainpages=false,
pdfpagelabels]{hyperref}
\usepackage{bookmark}
\usepackage{amsmath,amsthm,amssymb,color,verbatim,graphicx,fullpage,url}

\usepackage{mathtools}
\usepackage{thm-restate}

\usepackage[capitalise]{cleveref}

\def\1{\mathbf{1}} 
\def\0{\mathbf{0}}

\usepackage{amsmath,amsthm, amssymb,amscd,amstext,amsfonts}
\usepackage[utf8]{inputenc}
\usepackage{mathtools}
\usepackage{mathrsfs}
\usepackage{thmtools}
\usepackage{latexsym}
\usepackage{verbatim}
\usepackage{framed}
\usepackage{graphicx}
\usepackage{stmaryrd}
\usepackage{enumitem}
\usepackage{fullpage}
\usepackage{bm}
\usepackage{url}
\usepackage{physics}
\usepackage{dsfont}
\usepackage{todonotes}
\usepackage{dsfont}

\usepackage{mathtools}
\usepackage[normalem]{ulem}

\usepackage{color}

\theoremstyle{plain}

\newtheorem{thm}{Theorem}[section]
\newtheorem{theorem}[thm]{Theorem}

\newtheorem{lemma}[thm]{Lemma}
\newtheorem{proposition}[thm]{Proposition}
\newtheorem{definition}[thm]{Definition}

\newtheorem{claim}[thm]{Claim}
\newtheorem*{claim*}{Claim}

\theoremstyle{remark}
\newtheorem{remark}[thm]{Remark}

\renewcommand{\norm}[1]{\|#1\|}

\newcommand{\set}[1]{\{#1\}}
\newcommand{\Set}[1]{\left\{#1\right\}}

\def\supp{{\mathrm{supp}}}                               

\DeclareMathOperator{\Ex}{\mathbb{E}}           

\newcommand{\N}{\mathbb{N}}

\newcommand{\cI}{\mathcal I}

\newcommand{\eps}{\varepsilon}

\usepackage{ifthen}
\newcommand{\agnorm}[2][]{
	\ifthenelse{\equal{#2}{}}{
		\widetilde{\gamma}_2^{#1}
	}{
		\widetilde{\gamma}_2^{#1}(#2)
	}
}

\DeclareFontFamily{U}{mathx}{}
\DeclareFontShape{U}{mathx}{m}{n}{<-> mathx10}{}
\DeclareSymbolFont{mathx}{U}{mathx}{m}{n}
\DeclareMathAccent{\widecheck}{0}{mathx}{"71}

\newcommand{\cH}{\mathcal{H}}
\newcommand{\cA}{\mathcal{A}}
\newcommand{\cD}{\mathcal{D}}
 \newcommand{\cL}{\mathcal L}

\DeclareMathOperator{\ldim}{Ldim}
\DeclareMathOperator{\VCdim}{VCdim}
 
\DeclareMathOperator{\ord}{ord}

\usepackage{soul}

\begin{document}

\title{Stability and List-Replicability for Agnostic Learners}
\author{
    Ari Blondal\thanks{McGill University, \texttt{ari.blondal@mail.mcgill.ca}} \and
    Shan Gao\thanks{McGill University, \texttt{shan.gao5@mail.mcgill.ca}} \and
    Hamed Hatami\thanks{McGill University, \texttt{hatami@cs.mcgill.ca}. Supported by an NSERC grant.} \and
    Pooya Hatami\thanks{Ohio State University, \texttt{pooyahat@osu.edu}}
}

\maketitle

\begin{abstract}
Two seminal papers--Alon, Livni, Malliaris, Moran~(STOC 2019)  and Bun, Livni, and Moran~(FOCS 2020)--established the equivalence between online learnability and globally stable PAC learnability in binary classification.  However, Chase, Chornomaz, Moran, and Yehudayoff (STOC 2024) recently showed that this equivalence does not hold in the agnostic setting. Specifically, they proved that in the agnostic setting, only finite hypothesis classes are globally stable learnable. Therefore, agnostic global stability is too restrictive to capture interesting hypothesis classes. 

To address this limitation, Chase \emph{et al.} introduced two relaxations of agnostic global stability.  In this paper, we characterize the classes that are learnable under their proposed relaxed conditions, resolving the two open problems raised in their work.

First, we prove that in the setting where the stability parameter can depend on the excess error (the gap between the learner's error and the best achievable error by the hypothesis class), agnostic stability is fully characterized by the Littlestone dimension. Consequently, as in the realizable case, this form of learnability is equivalent to online learnability.

As part of the proof of this theorem, we strengthen the celebrated result of Bun \emph{et al.} by showing that classes with infinite Littlestone dimension are not stably PAC learnable, even if we allow the stability parameter to depend on the excess error.

For the second relaxation proposed by Chase \emph{et al.}, we prove that only finite hypothesis classes are globally stable learnable even if we restrict the agnostic setting to distributions with small population loss. 
\end{abstract}

\section{Introduction}

We follow the standard PAC learning framework for \emph{binary classification} as, for example, described in \cite{shalev2014understanding}. In this model, a learner receives a sample of i.i.d. \emph{examples} from an unknown distribution $\cD$ over $X \times \set{0,1}$, where $X$ is the domain set, and $\set{0,1}$ represents the two possible \emph{labels} in binary classification. The learner's goal is to produce a \emph{hypothesis} $h:X \to \set{0,1}$ that minimizes the \emph{population loss}
\[\cL_\cD(h) \coloneqq \Pr_{(\bm{x},\bm{y})\sim \cD}[h(\bm{x}) \neq \bm{y}].\] 
Here, and throughout the paper, we use boldface letters to denote random variables and use the notation $(\bm{x},\bm{y})\sim \cD$ to express that $(\bm{x},\bm{y})$ is a random variable distributed according to $\mathcal{D}$.  

Formally, a \emph{learning rule} is a (randomized) function $\bm{\cA}$ that maps any sample 
$ S \in (X \times \set{0,1})^* \coloneqq \bigcup_{n=0}^\infty (X \times \set{0,1})^n$ to a hypothesis $\bm{\cA}(S) \in \set{0,1}^X$.  Thus, for any given sample $S$, $\bm{\cA}(S)$ is a random variable taking values in $\set{0,1}^X$.   

Throughout this paper, all learning rules are assumed to be randomized. We consistently use $X$ to denote the domain, $\set{0,1}$ to represent the two possible labels, and $\cD$ always refers to a distribution over $X \times \set{0,1}$. For an integer $n>0$, we use $[n]$ to denote the set $\{1,\dots, n\}$.

Given a \emph{hypothesis class} $\cH \subseteq \set{0,1}^X$, the goal of PAC (Probably Approximately Correct) learning is for the learner to produce, with high probability, a hypothesis whose population loss is close to the best achievable within $\cH$, defined as 
\[\cL_\cD(\cH) \coloneqq \inf_{h\in \cH} \cL_\cD(h).\]
A class $\cH$ is \emph{PAC learnable} if there is a learning rule $\bm{\cA}$ and a function $n(\epsilon, \delta)$ such that for any $\epsilon, \delta > 0$,
\begin{equation}\label{eq:PAC_Learning}
    \Pr_{\bm{S} \sim \cD^n} [\cL_\cD(\bm{\cA}(\bm{S})) \leq \cL_\cD(\cH) + \epsilon] \geq 1 - \delta \ \text{ where } n=n(\epsilon,\delta).
\end{equation}
PAC learning is studied in the \emph{realizable case}, where we assume $\cL_\cD(\cH) = 0$, and the \emph{agnostic case}, where  $\cL_\cD(\cH) > 0$.

\paragraph{Replicability and Global Stability.} Replicability is a fundamental principle of the scientific method. A study is replicable if it consistently yields the same results when repeated with new data drawn from the same distribution or source. In recent years, machine learning has seen a growing need to address the replication crisis~\cite{ball2023ai, baker20161}.  Impagliazzo, Lei, Pitassi, and Sorrel~\cite{ILPS22} initiated a formal theoretical framework for studying replicability in machine learning. Since their work, a rapidly growing body of research has emerged that introduced various notions of replicability. These works and subsequent research showed that many of these notions of replicability are essentially equivalent. Furthermore, they established deep connections to other foundational concepts in learning theory, such as differential privacy~\cite{chase2023replicabilitystabilitylearning, bun2023stability, kalavasis2023statistical,ghazi2021user, chase2023local}.   Additionally, a growing body of work has explored replicability in many data analysis and learning settings~\cite{ILPS22, bun2023stability, karbasi2023replicability, esfandiari2023replicable, Esfandiarietal23, eaton2023replicable, kalavasis2024replicable, kalavasis2023statistical}.

In this paper, we focus on the notion of replicability where the learning algorithm is expected to often produce the same predictor when applied to two independent and identically distributed inputs. This concept was first introduced under the term \emph{global stability} in~\cite{BLM20} and has since been refined and explored in subsequent works~\cite{ghazi2021user, kalavasis2023statistical, chase2023replicabilitystabilitylearning, chase2023local}. We start by defining global stability. 



\begin{restatable}[$\rho$-Global Stability, \cite{chase2023local}]{definition}{stabledef}
\label{def:stability}
    Given a function $\rho:(0,1)\rightarrow (0,1)$, a learning rule $\bm{\cA}$ is a \emph{$\rho$-global stable learner} for a hypothesis class $\cH$ if the following holds. 
    For every $\epsilon>0$, there exists $n=n(\epsilon)$ such that for every realizable distribution $\cD$, there exists a hypothesis $h$ satisfying 
\[\cL_\cD(h)\leq \epsilon\] 
and 
\begin{equation}
\label{eq:stable}
\Pr_{\substack{\bm{S}\sim \cD^n }}[\bm{\cA}(\bm{S})=h] \geq \rho(\epsilon). 
\end{equation}

Similarly, we call $\bm{\cA}$ a \emph{$\rho$-global stable agnostic learner} for $\cH$, if there exists $n=n(\epsilon)$ such that for every distribution $\cD$ on $X\times \set{0,1}$, there exists a hypothesis $h \in \set{0,1}^X$ that satisfies \eqref{eq:stable} and \begin{equation}\label{eq:excesserror}
\cL_\cD(h)\leq \cL_\cD(\cH)+\epsilon.
\end{equation} 
\end{restatable}

To simplify terminology, we use the term \emph{$\rho$-global stable} to describe a hypothesis class $\cH$ with a \emph{$\rho$-global stable learner}. Likewise, we call $\cH$ \emph{agnostically $\rho$-global stable} if it has a $\rho$-global stable agnostic learner.

\begin{definition}[Global Stability, \cite{BLM20}]\label{def:globalstability}
We say that a hypothesis class $\cH$ is \emph{globally stable} if it is $\rho$-global stable for a fixed constant $\rho \in (0,1)$.  Similarly, a hypothesis class $\cH$ is \emph{agnostically globally stable} if it is agnostically $\rho$-stable for such a constant.  
\end{definition}
In short, global stability requires the stability parameter in~\Cref{eq:stable} to be uniform, meaning it must not depend on $\epsilon$.

Bun, Livni, and Moran~\cite{BLM20} showed that in the realizable setting, global stability is fully characterized by bounded Littlestone dimension (\Cref{def:LDim}). This result, combined with the seminal works of Littlestone~\cite{littlestone1988learning} and Alon \emph{et al.}~\cite{Alon_22_private_and_online}, shows that global stability is equivalent to online learnability and approximately private learnability, as well as some other notions of replicability~\cite{kalavasis2023statistical,ghazi2021user,bun2023stability}.

In contrast, the agnostic setting reveals a different picture.  Chase, Chornomaz, Moran, and Yehudayoff~\cite{chase2023local} proved the following characterization using a topological approach.  

\begin{theorem}[\cite{chase2023local}] 
\label{thm:Chase_ag}
A hypothesis class $\cH$ is agnostically globally stable if and only if $\cH$ is finite. 
\end{theorem}
This striking result shows that agnostic global stability is far more restrictive than its realizable counterpart.  Since finite classes are trivially global stable, \Cref{thm:Chase_ag} shows that agnostic global stability is too restrictive to lead to interesting learnability phenomena. To remedy this, Chase \emph{et al.}~\cite{chase2023local} introduced two relaxations of agnostic global stability and proposed a study of which hypothesis classes can be learned under these relaxed notions of stability.

\paragraph{Excess-error dependent stability.}  The first suggested relaxation, coincides with our definition of $\rho$-global stability in \Cref{def:stability}. A hypothesis class $\cH$ is called \emph{excess-error dependent stable} if it is agnostically $\rho$-global stable for some $\rho:(0,1) \to (0,1)$. Here, the excess-error refers to the parameter $\epsilon$ in \cref{eq:excesserror}.

Our main theorem provides a complete characterization of such classes. We show that a hypothesis class is agnostically $\rho$-global stable learnable for some $\rho$ if and only if it has a bounded Littlestone dimension. 
We denote the Littlestone dimension of $\cH$ as $\ldim(\cH)$.

\begin{restatable}[Main Theorem]{theorem}{mainthm}
\label{thm:main}
Let $\cH$ be a binary concept class. 
\begin{itemize}
    \item[(i)] If $\ldim(\cH)=\infty$, then $\cH$ is not $\rho$-global stable for any $\rho:(0,1) \to (0,1)$.
    \item[(ii)] If $\ldim(\cH)<\infty$, then $\cH$ is agnostically $\rho$-global stable for some $\rho:(0,1) \to (0,1)$.
\end{itemize}
\end{restatable}

Note that \Cref{thm:main}~(i) states that if $\ldim(\cH)=\infty$, then even in the realizable case, we cannot achieve $\rho$-global stability for any $\rho:(0,1) \to (0,1)$. This strengthens the result of Bun, Livni, and Moran~\cite{BLM20}, which only overrules $\rho$-global stability when $\rho > 0$ is a fixed constant. 

Shortly after a draft of this paper was posted online, Hopkins and Moran~\cite{hopkins25} communicated to us that in an independent work, they have proved an equivalent statement to Theorem 4
by utilizing the known relation between stability and differential privacy. They use a $\rho$-global stable learner to achieve weak DP learning, which in turn is boosted to a strong DP learner. It is well known that strong DP learning is achievable if and only if the Littlestone dimension is finite~\cite{Alon_22_private_and_online}. 
In contrast, our proof is direct and relies solely on notions of stability and list-replicability.

Combined with the work of Alon \emph{et al.}~\cite{Alon_22_private_and_online}, \Cref{thm:main} implies that agnostic $\rho$-global stability is equivalent to global stability, as well as to approximate private learnability and online learnability.

\paragraph{Class-error dependent stability.}  In many practical learning scenarios, while we cannot assume realizability, we may have prior knowledge that the hypothesis class performs reasonably well. This corresponds to a more restricted version of agnostic learning, where the learning task is limited to distributions $\cD$ that satisfy $\cL_\cD(\cH) \leq \gamma$ for some small $\gamma>0$.

\begin{definition}[Class-error Dependent Stability, \cite{chase2023local}]
\label{def:class_error}
Let $\gamma \in [0,1]$ be a fixed constant. We say $\cH$ is $\gamma$-agnostically globally stable if there exists a  constant $\rho>0$ and a learning rule $\bm{\cA}$ such that the following holds. For every $\epsilon>0$, there exists $n=n(\epsilon)$ such that for every distribution $\cD$ with $\cL_\cD(\cH)\leq \gamma$, there exists a hypothesis $h$ satisfying  
\[\cL_\cD(h)\leq \cL_\cD(\cH)+\epsilon,\]
and
\[ 
\Pr_{ \bm{S} \sim \cD^n} \left[\bm{\cA} (\bm{S})=h \right] \geq \rho. 
\] 
\end{definition}

The case $\gamma=0$ corresponds to the realizable case, where Bun, Livni, and Moran~\cite{BLM20} show that global stability is fully characterized by bounded Littlestone dimension. On the other hand, $\gamma=1$ corresponds to the agnostic case, where \Cref{thm:Chase_ag} shows that only finite classes are agnostically globally stable.

Chase \emph{et al.}~\cite{chase2023local} ask which hypothesis classes are $\gamma$-agnostically globally stable for all sufficiently small $\gamma$. Our next theorem shows that the realizable case, $\gamma=0$, is the only scenario in which infinite hypothesis classes can be $\gamma$-agnostically globally stable. 
Therefore, the relaxation of agnostic global stability to $\gamma$-agnostic global stability does not lead to any generalization, as only finite hypothesis classes can be $\gamma$-agnostically globally stable if $\gamma > 0$. 

To prove our theorem, we show that agnostic global stability reduces to $\gamma$-agnostic global stability, for any arbitrary $\gamma>0$. 

\begin{theorem}\label{thm:class-error-dependent}
If a class $\cH \subseteq \set{0,1}^X$ is $\gamma$-agnostically globally stable for some $\gamma>0$, then $\cH$ is finite. 
\end{theorem}
\begin{proof}
Assume towards a contradiction that an infinite $\cH \subseteq \set{0,1}^X$ is $\gamma$-agnostically globally stable for some $\gamma>0$, and let $\rho>0$, $\bm{\cA}$, and $n(\cdot)$ be as in \Cref{def:class_error}. 

Pick any $x^* \in X$, and let $b^*\in \set{0,1}$ be such that the subclass $\cH^* \coloneqq \set{h \in \cH: h(x^*) = b^*}$ is infinite.  Let $\gamma' \coloneqq \min \set{\gamma,\frac{1}{10}}$. We obtain a contradiction with \Cref{thm:Chase_ag} by showing that $\cH^*$ is agnostically globally stable despite being infinite.

Given $\epsilon>0$ and access to a distribution $\cD$ on $X \times \set{0,1}$, let $n \coloneqq n(\epsilon \gamma' )$, and define the distribution 
\[ \cD' \coloneqq \gamma' \cD + (1-\gamma') \1_{(x^*,b^*)},\] 
which corresponds to sampling from $\cD$ with probability $\gamma'$, and sampling $(x^*,b^*)$ with probability $1-\gamma'$. Consider the learning rule $\bm{\cA}'$ described as follows: 

    \begin{enumerate}
        \item Given a sample $\bm{S} \sim \cD^n$, independently replace each example in $\bm{S}$ with $(x^*,b^*)$ with probability $1-\gamma'$. Let $\bm{T}$ denote the resulting modified sample. 
        \item Output $\bm{\cA}(\bm{T})$.
    \end{enumerate}
Note that $\bm{\cA}'(\bm{S})$ with $\bm{S} \sim \cD^n$ has the same distribution as $\bm{\cA}(\bm{T})$ with $\bm{T} \sim (\cD')^n$. 

Since every $h \in \set{0,1}^X$ with $h(x^*)=b^*$, satisfies $\cL_{\cD'}(h) = \gamma' \cL_{\cD}(h) \le \gamma'$, we have $\cL_{\cD'}(\cH) \le \gamma'$. Therefore, by our choice of  $n \coloneqq n(\epsilon \gamma' )$ and our assumption of the $\gamma$-agnostic global stability of $\bm{\cA}$ on 
$\cH$,  there exists $h^* \in \set{0,1}^X$ with 
\begin{equation}
\label{eq:class-error}
\cL_{\cD'}(h^*)\leq \cL_{\cD'}(\cH)+  \epsilon\gamma' \ \ \text{ and } \ \  \Pr_{\bm{S} \sim \cD^n} \left[\bm{\cA}' (\bm{S})=h^* \right] = \Pr_{ \bm{T} \sim (\cD')^n} \left[\bm{\cA} (\bm{T})=h^* \right]\geq \rho. 
\end{equation}
If $h^*(x^*) \neq b^*$, then since $\gamma' < \frac{1}{10}$, we  have
\[\cL_{\cD'}(h^*) \ge 1-\gamma' >\gamma'+  \epsilon\gamma'  \ge \cL_{\cD'}(\cH)+ \epsilon\gamma',\]
which contradicts the first inequality in \Cref{eq:class-error}.  Therefore, $h^*(x^*)=b^*$, and consequently, we have $\cL_{\cD'}(h^*)=\gamma' \cL_{\cD}(h^*)$. Furthermore, $\cL_{\cD'}(\cH)=\gamma' \cL_{\cD}(\cH^*)$.   Replacing these in \Cref{eq:class-error} shows 
\[\cL_{\cD}(h^*)\leq \cL_{\cD}(\cH^*)+  \epsilon \ \ \text{ and } \ \  \Pr_{\bm{S} \sim \cD^n} \left[\bm{\cA}' (\bm{S})=h^* \right]  \geq \rho. 
  \]
Therefore, the infinite class $\cH^*$ is agnostically globally stable, contradicting  \Cref{thm:Chase_ag}.
\end{proof}

\paragraph{Relation to list replicability.}

In learning theory, global stability is more useful when paired with a guarantee that the learner typically outputs a hypothesis with a low population loss.   The definition of global stability (\Cref{def:globalstability}) only requires that the learner outputs a low-error hypothesis with some probability $\rho > 0$, and the learner can output hypotheses with large population loss with probability $1-\rho$. However,  it is known that global stability implies bounded VC dimension~\cite{Alon_22_private_and_online}, and assuming bounded VC dimension, this can be easily remedied. The learner can estimate the population loss of its output by comparing it to that of the hypothesis produced by the empirical risk minimization (ERM) rule. If the population loss is unsatisfactory, the learner can fall back on the ERM hypothesis instead.

Next, we introduce a seemingly stronger notion of replicability, originally proposed by Chase \emph{et al.}~\cite{chase2023replicabilitystabilitylearning}, known as \emph{list replicability}. This notion strengthens the standard guarantee by requiring that, with high probability, the output hypothesis lies within a small list of hypotheses, each with low population loss.


\begin{definition}[List Replicability, \cite{chase2023replicabilitystabilitylearning}]
\label{def:List}
Given a function $L: (0,1) \rightarrow \N$, we say that a learner $\bm{\cA}$ is an \emph{$L$-list-replicable learner} for a hypothesis class $\cH$ if the following holds. For every $\epsilon,\delta>0$, there exists $n=n(\epsilon, \delta)$  such that for every realizable distribution $\cD$,  there exists a list of $L=L(\epsilon)$ hypotheses $h_1,\ldots, h_L$ satisfying 
\[\cL_\cD(h_i)\leq \epsilon\ \text{ for all } 1\le i \le L\] 
and 
\begin{equation}
    \label{eq:likely_in_list}
    \Pr_{\bm{S}\sim \cD^n}[\bm{\cA}(\bm{S}) \in \{h_1, \ldots, h_L\}] \geq 1 - \delta. 
\end{equation}

Similarly, $\bm{\cA}$ is an \emph{agnostic $L$-list-replicable learner} for $\cH$ if there exists $n = n(\epsilon, \delta)$ such that for every distribution $\cD$ on $X \times Y$, there exists a list of $L=L(\epsilon)$ hypotheses $h_1, \ldots, h_L \in \cH$ satisfying \eqref{eq:likely_in_list} and 
\[ \cL_\cD(h_i) \leq \cL_\cD(\cH) + \epsilon \ \text{ for all } 1\le i \le L.\] 
\end{definition}

Similar to stability, we define the notions of \emph{global list-replicability} and \emph{agnostic global list-replicability} to describe the uniform case where the learner in \Cref{def:List} exists for a fixed constant $L>0$ independent of $\epsilon$.\footnote{In the literature, what we refer to as global list-replicability is simply called list-replicability.}  

It is worth noting that \Cref{eq:likely_in_list} easily implies stability, as there must exist some  $i\in  [L]$ with 

\[ \Pr_{\bm{S}\sim \cD^n}[\bm{\cA}(\bm{S}) =h_i] \ge \frac{1-\delta}{L}.\]

Chase, Moran, and Yehudayoff~\cite{chase2023replicabilitystabilitylearning} showed that the converse is also true: global stability implies global list-replicability. 

\begin{theorem}[\cite{chase2023replicabilitystabilitylearning}]
\label{thm:List_vs_Stab}
For any fixed constant $L>0$, a hypothesis class is $L$-list-replicable if and only if it is $\rho$-global stable for all $\rho<\frac{1}{L}$. 
\end{theorem}

Analogous to \Cref{def:class_error}, given a parameter $\gamma \in [0,1]$, we refer to a class $\cH$ as $\gamma$-agnostically list-replicable if we relax the requirement of the agnostic global list-replicability to only consider distributions $\cD$ with $\cL_\cD(\cH)\leq \gamma$. 

The foregoing relaxations of global list-replicability were proposed in \cite{chase2023local}, where they asked for a characterization of hypothesis classes that can be agnostically learned under these notions, termed excess-error dependent and class-error dependent list-replicability.

Our next theorem extends \Cref{thm:List_vs_Stab} to show that these new notions coincide with their stability counterparts. Consequently, \Cref{thm:main} and \Cref{thm:class-error-dependent} completely resolve the questions posed in \cite{chase2023local}.
\begin{restatable}{theorem}{propstabilityvslistrep}
\label{thm:stability-vs-listrep}
    Consider a parameter $\gamma \in [0,1]$. 
\begin{itemize}
\item[(i)] Given $L:(0,1) \to \mathbb{N}$, if a class $\cH$ is  $\gamma$-agnostically $L$-list replicable, then it is $\gamma$-agnostically $\rho$-global stable for any $\rho:(0,1) \to [0,1]$ satisfying $\rho(\epsilon)<\frac{1}{L(\epsilon)}$ for all $\epsilon \in (0,1)$.
\item[(ii)] Given $\rho:(0,1) \to (0,1]$, if a class $\cH$ is $\gamma$-agnostically  $\rho$-global stable, then it is   $\gamma$-agnostically $L$-list replicable, for $L(\epsilon) \coloneqq \left\lfloor \frac{1}{\rho(\epsilon/4)} \right\rfloor$.
\end{itemize}
 \end{restatable}

\subsection{Preliminaries: VC Dimension, Uniform Convergence, and Littlestone Dimension}
This section outlines a few key concepts and results from learning theory. More specifically, we state the connections between PAC learnability, VC dimension, and uniform convergence, and we state the definition of the Littlestone dimension. For a detailed exposition, see \cite{shalev2014understanding}.

A fundamental result of learning theory is that a class $\cH$ is PAC-learnable if and only if it satisfies the \emph{Uniform Convergence} property. For a sample of $m$ examples $S \in (X \times \set{0,1})^m$, and a hypothesis $h:X \to \set{0,1}$, let 
\[ \cL_{\bm{S}}(h) \coloneqq \Pr_{(\bm{x},\bm{y}) \sim S} [h(\bm{x}) \neq \bm{y}],  \] 
denote the \emph{empirical population loss} of $h$ with respect to $S$. 

\begin{definition}[Uniform Convergence]
A binary hypothesis class $\cH$ has the \emph{Uniform Convergence} property if, for any $\epsilon, \delta \in (0, 1)$, there exists $n(\epsilon, \delta)$ such that for any distribution $\cD$, we have  
 \begin{align*}
    \Pr_{\bm{S} \sim \cD^n} \left[ |\cL_{\bm{S}}(h) - \cL_\cD(h)| < \epsilon \text{ for all } h \in \cH \right] \geq 1-\delta.
 \end{align*}
\end{definition}

The fundamental theory of PAC learning states that the Uniform Convergence property and, consequently, PAC-learnability are characterized by having a finite Vapnik-Chervonenkis (VC) dimension. 

\begin{definition}[VC dimension]
The VC dimension of a binary hypothesis class $\cH$ is the size of the largest subset $X'$ of $X$ such that, for every binary labelling of $X'$, there is a hypothesis $h \in \cH$ consistent with that labelling. Such a set $X'$ is said to be \emph{shattered}  by $\cH$. If arbitrarily large sets can be shattered, the VC dimension is defined to be $\infty$. 
\end{definition}

The Littlestone dimension relaxes the VC dimension by shattering decision trees instead of sets.   A \emph{mistake tree} of depth $d$ over a domain $X$ is a complete binary tree of depth $d$ with the following properties:
\begin{itemize}
\item Each internal node in the tree is labelled by an element $x \in X$.
\item Each edge is labeled by a binary value  $b \in \set{0,1}$ where $b=0$ indicates a left child and $b=1$ indicates a right child. 
\end{itemize}
Every \emph{root-to-leaf path} in the tree is described by a sequence  $(x_1,b_1),\ldots,(x_d,b_d)$ where $x_i \in X$ is the label of the $i$th internal node on the path and $b_i$ specifies whether the path moves to the left or right child at each level.
 
We say that a mistake tree is \emph{shattered} by a hypothesis class $\cH \subseteq \set{0,1}^X$ if for every root-to-leaf path $(x_1,b_1),\ldots,(x_d,b_d)$ where $x_i \in X$ and $b_i \in \set{0,1}$, there exists a hypothesis $h \in \cH$ with $h(x_i)=b_i$ for all $i \in [d]$. 

\begin{definition}[Littlestone Dimension]
 \label{def:LDim}
The Littlestone dimension of a hypothesis class $\cH$, denoted $\ldim(\cH)$, is the largest integer $d$ such that there exists a mistake tree of depth $d$ shattered by $\mathcal{H}$.  
\end{definition}

We always have $\VCdim(\cH) \le \ldim(\cH)$, since every shattered set $X'=\set{x_1,\ldots,x_d}$ gives rise to a mistake tree of depth $d$ where all nodes at level $i$ are labelled with $x_i$.  This tree is clearly shattered by $\cH$.

\section{\Cref{thm:stability-vs-listrep}: Stability and List-Replicability are equivalent}\label{subsec:equivalence_stable_list}

In this section, we prove \Cref{thm:stability-vs-listrep}, which establishes the equivalence between stability and list replicability. 
This result generalizes the equivalence between global stability and global list replicability of~\cite{chase2023replicabilitystabilitylearning}.

For the reader's convenience, we recall the statement of the theorem and provide a summary of the proof.

\propstabilityvslistrep*

To prove (ii), we construct an agnostic $L$-list replicable learner by running the $\rho$-global stable learning algorithm multiple times.
We return any output hypothesis whose empirical loss is close to that of the best output hypothesis, and whose empirical frequency is not much smaller than $\rho$. This guarantees that we typically output a hypothesis with low population loss and high likelihood of being an output of the globally stable learner. The latter ensures that our output is typically confined to a small list.
\begin{proof}
    The proof of (i) is straightforward. Given $\eps>0$, let $\delta >0$ be arbitrary and let $n=n(\epsilon,\delta)$ be the sample complexity of a $\gamma$-agnostic $L$-list-replicable learner for $\cH$.  Let $\cD$ be a distribution with population  loss at most $\gamma$, and let $h_1,\ldots, h_{L(\epsilon)}$ be the list of hypotheses satisfying \Cref{eq:likely_in_list}. At least one of these hypotheses $h_i$ satisfies     
    \[ \Pr_{\bm{S}\sim \cD^n}[\bm{\cA}(\bm{S}) =h_i] \ge \frac{1-\delta}{L(\epsilon)}\ge\frac{1}{L(\epsilon)}-\delta.\]  
    Since this statement holds for every $\delta>0$,  $\cH$ is $\gamma$-agnostically $\rho$-global stable for all  $\rho(\epsilon)<\frac{1}{L(\epsilon)}$. 

   To prove (ii), consider an $\epsilon>0$, and let $\delta > 0$ be any confidence parameter.  For the sake of brevity, denote $\rho\coloneqq \rho(\epsilon/4)$ and $L \coloneqq L(\epsilon)= \left\lfloor \frac{1}{\rho(\eps/4)}\right\rfloor$. Thus, we have $\rho  \in \left(\frac{1}{L +1}, \frac{1}{L }\right]$. Let
    \[ 
        \alpha \coloneqq \rho - \frac{1}{L+1} > 0.
    \] 
Let $n_0=n_0(\rho,\epsilon)$ 
be sufficiently large such that the global stability property holds, namely, for every $\cD$ with population  loss at most $\gamma$, there exists $h^*:X \to \set{0,1}$ satisfying
\begin{equation}
\label{eq:stable_equiv}
\cL_\cD(h^*) \le \cL_\cD(\cH)+\frac{\epsilon}{4} \ \text{ and }\ \Pr_{\substack{\bm{S}\sim \cD^{n_0} }}[\bm{\cA}(\bm{S})=h^*] \geq \rho. 
\end{equation}    
    For any $h \in \set{0,1}^X$, define  
    \[
        p(h) \coloneqq \Pr_{\bm{S} \sim \cD^{n_0}}[\bm{\cA}(\bm{S}) = h], 
    \] 
    and consider    
    \[\Lambda \coloneqq \Set{h \in \set{0,1}^X : p(h) > \frac{1}{L+1} \text{ and } \cL_{\cD}(h) \le \cL_\cD(\cH)+\epsilon}. \]
    Note that $|\Lambda| \le L$, and $\Lambda$ is nonempty, as it contains $h^*$. It suffices to design a learning rule $\bm{\cA}'$ such that with probability at least $1-\delta$, it outputs a hypothesis from $\Lambda$.

 Since $\VCdim(\cH) < \infty$, by the uniform convergence property of $\cH$, there exists $n_1 \in \mathbb{N}$ such that for any distribution $\cD$, 
    \begin{equation}
    \label{eq:uniform_convergence}
        \Pr_{\bm{Q} \sim \mathcal{D}^{n_1}} \left[ \sup_{h \in \mathcal{H}} |\mathcal{L}_{\bm{Q}}(h) - \mathcal{L}_\mathcal{D}(h)| \le  \frac{\epsilon}{4} \right] \ge 1- \frac{\delta}{2}.
    \end{equation}

    Let $t\coloneqq t(\alpha,\delta)$ be a sufficiently large integer to be determined later.   We propose the following learning rule $\bm{\cA}'$ with sample complexity $t n_0+n_1$:
    \begin{enumerate}
        \item Sample $\bm{S}=(\bm{P},\bm{Q}) \sim \cD^{t n_0+n_1}$, where $\bm{P}=(\bm{P}_1,\ldots, \bm{P}_t) \sim (\cD^{n_0})^t = \cD^{t n_0}$ and $\bm{Q} \sim \cD^{n_1}$.

        \item For every $i \in [t]$, let $\bm{h}_i = \bm{\cA}(\bm{P}_i)$. Define the empirical estimate of $p(h)$ as
\[        
\widehat{p}_{\bm{S}}(h)\coloneqq \frac{|\set{i\in [t] \ | \ \bm{h}_i=h}|}{t}.
\]  
 
\item Output any hypothesis $h \in  \set{0,1}^X$ that satisfies the following two conditions. 
  
        \begin{enumerate}
            \item $\widehat{p}_{\bm{S}}(h)\geq \rho - \frac{\alpha}{2}$;

            \item $\mathcal{L}_{\bm{Q}}(h) \leq \inf_{h' \in \cH} \mathcal{L}_{\bm{Q}}(h') + \frac{3\epsilon}{4}$;            
        \end{enumerate}
     If no such $h$ exists, output an arbitrary $h$ corresponding to ``failure''.
    \end{enumerate}
    
    We show that $\bm{\cA}'$ is a $\gamma$-agnostic $L$-list replicable learner with error at most $\epsilon$.
    Let $\mathcal{D}$ be any distribution with population loss at most $\gamma$.    

\begin{claim}
We have  
\begin{equation}
\label{eq:indicator_uniform}
        \Pr_{\bm{S} \sim \mathcal{D}^{t n_0+n_1}}
        \left[
            |p(h) - \widehat{p}_{\bm{S}}(h)| < \frac{\alpha}{2}
            \text{ for all } h \in \set{0,1}^X 
        \right] \geq 1 - \frac{\delta}{2}.
\end{equation}    
\end{claim}
\begin{proof}
We use the uniform convergence property of the family of indicator functions on $\cH$. More precisely, for $f \in \set{0,1}^X$, define $\1_f:\set{0,1}^X \to \set{0,1}$ as 
\[
\1_f(f') \coloneqq 
\begin{cases}
1 & f'=f \\
0 & \text{otherwise}
\end{cases}.
\]
The class 
\[\cI \coloneqq \set{\1_f \ : \ f \in \set{0,1}^X} \]
has VC dimension $1$, and therefore, it satisfies the uniform convergence property. For $\bm{S} \sim \cD^{t n_0+n_1}$,  $\bm{\cA}(\bm{S})$ induces a probability distribution $\mu$ on $\set{0,1}^X$, and we have
\[1- p(h)=\Pr_{\bm{S} \sim \cD^{n_0}}[\bm{\cA}(\bm{S}) \neq h] = \cL_\mu(\1_h), \]
while $1-\widehat{p}_{\bm{S}}(h)$ corresponds to the empirical loss of  $(\1_{\bm{h}_1}, \ldots, \1_{\bm{h}_t}) \sim \mu^t$. By the uniform convergence property for $\cI$, for sufficiently large $t=t(\alpha,\delta)$, \Cref{eq:indicator_uniform} holds.  
\end{proof}
The following claim completes the proof. 
\begin{claim}
Consider $\bm{S}=(\bm{P},\bm{Q}) \sim \cD^{t n_0+n_1}$ and let $\bm{h}=\bm{\cA}'(\bm{S})$. 
\[\Pr[\bm{h} \in \Lambda] \ge 1-\delta. \]
\end{claim}
\begin{proof}
By \Cref{eq:uniform_convergence,eq:indicator_uniform} and the union bound, with probability at least $1-\delta$ ,  we have  
\[ |\mathcal{L}_{\bm{Q}}(h) - \mathcal{L}_\mathcal{D}(h)| \le  \frac{\epsilon}{4} 
            \text{ for all } h \in \cH, \] 
and 
\[ |p(h) - \widehat{p}_{\bm{S}}(h)| < \frac{\alpha}{2}
            \text{ for all } h \in \set{0,1}^X.\]
Let $\mathcal{E}$ denote the event that $\bm{S}$ satisfies both these statements.  Conditioning on $\mathcal{E}$, we have 
\begin{equation}
\label{eq:emp_real}
\left|\inf_{h' \in \cH} \mathcal{L}_{\bm{Q}}(h') - \cL_{\cD}(\cH)\right| \le \frac{\epsilon}{4},
\end{equation}
and any $h \in \set{0,1}^X$ satisfying Conditions ~3(a) and 3(b) satisfies
\[p(h) \ge \rho-\frac{\alpha}{2}-\frac{\alpha}{2} > \frac{1}{L+1}.\] 
and 
\[\cL_{\cD}(h) \le \inf_{h^\prime \in \mathcal{H}} \cL_{\bm{Q}}(h^\prime)+\frac{3\epsilon}{4}\leq \cL_\cD(\cH)+\epsilon.\] 
Therefore, all such $h$ belong to $\Lambda$. 

Finally, let $h^*$ be the hypothesis from \Cref{eq:stable_equiv}. We have $\widehat{p}_{\bm{S}}(h^*) > \rho - \frac{\alpha}{2}$ and  
\[\mathcal{L}_{\bm{Q}}(h^*) \le \cL_\cD(h^*)+  \frac{\epsilon}{4}  \le \cL_\cD(\cH)+\frac{\epsilon}{4} +  \frac{\epsilon}{4} \le \inf_{h' \in \cH} \mathcal{L}_{\bm{Q}}(h') + \frac{3\epsilon}{4}. \]
Therefore, $h^*$ satisfies Conditions~3(a)~and~(b), and the output of $\bm{\cA}'$ will not correspond to ``failure''.  
\end{proof}
\end{proof}

\section{\Cref{thm:main}~(i):  Stability implies finite Littlestone dimension} 

By the equivalence of global stability and list-replicability established in \Cref{thm:stability-vs-listrep}, \Cref{thm:main}~(i) is equivalent to the following theorem. 

\begin{theorem}
\label{thm:main_infinite}
If $\ldim(\cH) = \infty$, then $\cH$ is not $L$-list replicable for any $L:(0,1) \to \mathbb{N}$. 
\end{theorem}

The rest of this section is devoted to the proof of \Cref{thm:main_infinite}, which uses a classical result of Shelah~\cite{shelah1978classification} connecting the Littlestone dimension to the threshold dimension. 

\begin{definition}[Threshold dimension]
\label{def:thresholdDim}
The \emph{threshold dimension} of $\mathcal{H} \subseteq \set{0,1}^X$  is the largest $k$ such that there exists a set of inputs $\set{x_1,\ldots, x_k} \subseteq X$ and classifiers $\set{h_1,\ldots,h_k} \subseteq \mathcal{H}$ satisfying 
\[ h_t(x_i) = 1  \iff i \ge  t  \qquad \text{ for all } i,t \in [k]. \] 
\end{definition}

 We refer the reader to~\cite{alon2019private} for an accessible proof of the following result of Hodges~\cite{hodges1997shorter}, which provides effective bounds for a qualitative result of Shelah~\cite{shelah1978classification}. Shelah proved that any class $\cH$ with infinite Littlestone dimension also has an infinite threshold dimension. 
 
\begin{proposition}[\cite{hodges1997shorter}]
    \label{prop:ldim_matrix}
    If $\mathcal{H} \subseteq \set{0,1}^X$ has $\ldim(\mathcal{H}) = d$, its threshold dimension is at least $ \lfloor \log d \rfloor$. 
\end{proposition}

Throughout the proof, we will use the following observation stating that without loss of generality we may ignore the order of examples in a sample. 

\begin{remark}
\label{rem:symmetry}
Since the population loss does not depend on the order of the examples in $S$, and the examples are drawn independently from $\cD$, in the context of PAC learning and stability, we may assume that the learning rule disregards the order of the examples in any sample $S \in (X \times \set{0,1})^n$. In other words, the learning rule is invariant under the permutations of the examples in any given sample.  As a result, we often treat a sample $S$ as a multiset rather than a sequence. 
\end{remark} 

Next, we prove a lemma to decrease the probability of failure (i.e., $\delta$) in the definition of list replicability to a small function of the sample size $n$ and the population regret $\epsilon$.

\begin{lemma}[Boosting success probability]
\label{lem:list_samplesize}
Suppose $\cH$ is $L$-list-replicable for some $L:(0,1) \to \mathbb{N}$. For every  $C>1$, there exists a learning rule $\cA$ and a sample complexity $n_C:(0,1) \to \mathbb{N}$ such that the following holds. For every $\epsilon>0$ and every realizable distribution $\cD$,  there exists a list of $L=L(\epsilon)$ hypotheses $h_1,\ldots, h_L$ satisfying 
\[\cL_\cD(h_i)\leq \epsilon \text{ for all } 1\le i \le L\] 
and 
\[ \Pr_{\bm{S}\sim \cD^n}[\bm{\cA}(\bm{S}) \in \{h_1, \ldots, h_L\}] \geq 1 -\frac{\epsilon}{n^C} \ \ \text{ where } n=n_C(\epsilon). \] 
\end{lemma}
\begin{proof}
Define $\delta_0 \coloneqq \frac{1}{16L}$. By our assumption, there exists $n_0=n_0(\epsilon)$ and a learning rule $\cA'$ such that for any realizable distribution $\cD$, there exists a list $h_1, \dots, h_{L(\epsilon)}$ of hypotheses satisfying $\cL_\cD(h_i)\leq \epsilon$ for all $i$ and 
\begin{equation}
\label{eq:boost_list}
    \Pr_{\bm{S} \sim \cD^{n_0}} [\bm{\cA}'(\bm{S}) \in \{h_1, \dots, h_{L}\}] \geq 1-\delta_0=1 - \frac{1}{16L}.
\end{equation}
Since $\delta_0$ is fixed, $n_0$ depends only $\epsilon$.

Let $k>0$ be an integer to be determined later. We define a new learning rule $\cA$ that uses samples of size $kn_0$.  Given a sample $S=(S_1, \ldots, S_k) \in ((X \times \set{0,1})^{n_0})^{k}$, the learner $\cA$ outputs the most frequent hypothesis produced by the $k$ independent runs $\cA'(S_1),\ldots,\cA'(S_k)$. 

Let $\cD$ be any realizable distribution, and let $h_1, \dots, h_{L(\epsilon)}$ be as above. By \Cref{eq:boost_list}, there exists some $j^* \in [L]$ such that 
\[     \Pr_{\bm{S} \sim \cD^{n_0}} [ \cA'(\bm{S}) = h_{j^*}] \geq \frac{1}{2L}. \] 
Consider $\bm{S}=(\bm{S}_1, \ldots, \bm{S}_k) \sim (\cD^{n_0})^k$. For every $i \in [k]$, define the indicator variable $\bm{E}_i$ and $\bm{B}_i$ as 
\begin{itemize}
\item $\bm{E}_i=1$ iff $\cA'(\bm{S}_i)=h_{j^*}$;
\item $\bm{B}_i=1$ iff $\cA'(\bm{S}_i) \notin \set{h_1, \ldots, h_L}$.
\end{itemize}
 
The variables $\bm{E}_1, \ldots, \bm{E}_k$ are independent Bernoulli variables with $\Ex[\bm{E}_i] \geq \frac{1}{2L}$. Similarly, $\bm{B}_1, \ldots, \bm{B}_k$ are independent Bernoulli variables with $\Ex[\bm{B}_i]\leq \delta_0 \leq \frac{1}{16L}$. Define $\bm{E} \coloneqq \sum \bm{E}_i$ and $\bm{B} \coloneqq \sum \bm{B}_i$.
Applying Hoeffding's inequality, we have  
\[ 
    \Pr[\bm{E} \geq \frac{k}{4L}] \geq 1 - \Pr[|\bm{E}-\Ex[\bm{E}]|\geq \frac{k}{4L}] \geq 1 - e^{-\Omega(k/L^2)} \] 
and 
\[    \Pr[\bm{B} \leq \frac{k}{8L}] \geq 1 - \Pr[|\bm{B}-\Ex[\bm{B}]|\geq \frac{k}{16L}] \geq 1- e^{-\Omega(k/L^2)}.\] 

When both events occur, the output of $\cA$, the most frequent hypothesis, must come from the list $\set{h_1,\ldots, h_L}$. We may now choose $k$ such that $\delta \coloneqq 2\cdot e^{-\Omega(k/L^2)} \leq \frac{\epsilon}{(n_0 k)^C} = \frac{\epsilon}{n^C}$, concluding the proof.
\end{proof}

Let us give an overview of the proof of \Cref{thm:main_infinite}. By \Cref{prop:ldim_matrix}, it is sufficient to prove that hypothesis classes of infinite threshold dimension are not $L$-list replicable. Assume towards contradiction that there exists an $ L$-list-replicable learner for a hypothesis class of infinite threshold dimension. 

Similarly to Alon \emph{et al.}~\cite{alon2019private}, we use a hypergraph Ramsey argument to restrict the learning problem to an arbitrarily large subset $X$ of the domain, on which the learner's prediction is essentially determined by the ordered sign pattern of the sample. In particular, the Ramsey argument ensures that to label an element $x \in X$, the learner essentially looks at the ordered sign pattern of the labelled points in its given sample, and also at the ``order'' of $x$ with respect to these points.

We consider the distribution that is uniform over the threshold about the median of $X$.
Since the learner is a PAC learner, we show that the probability that its output classifies a point $x$ as a $1$ ranges from close to $0$ to close to $1$, as the order of $x$ ranges from the smallest to the largest element compared to the points in the samples. Consequently, we detect a probability jump from some order to the next.

The most significant difference between our proof and the proof of~\cite{alon2019private} lies in handling this probability jump. \cite{alon2019private} exploits the jump to create a ``privacy leak,'' whereas, without a privacy guarantee, we take a different approach based on an ``approximate rank'' argument.    More specifically, we use our boosting lemma (\Cref{lem:list_samplesize}) to show that for many samples, the function corresponding to the output probabilities can be well-approximated (in the $L_\infty$ norm) by a convex combination of a fixed short list of hypotheses. We then apply a ``volume-based'' argument to derive a contradiction by finding two samples that are well-approximated by the same convex combination, but are supposed to label a point $x \in X$ differently according to the aforementioned probability jump.

We are ready to present the proof of \Cref{thm:main_infinite}.

\begin{proof}[Proof of \Cref{thm:main_infinite}]
Fix some $\epsilon \in (0,1)$, and towards a contradiction, assume that $\ldim(\cH) = \infty$ and $\cH$ is $L$-list replicable for some $L:(0,1) \to \mathbb{N}$.

By \Cref{lem:list_samplesize}, there exists a constant $n$ and a learning rule $\cA$ such that for every realizable distribution $\cD$,  there exists a list of hypotheses $h_1,\ldots, h_L$ satisfying
\begin{equation}
\label{eq:main_thm_list1}
\cL_\cD(h_i)\leq \epsilon \text{ for all } 1\le i \le L
\end{equation}
and 
\begin{equation}
\label{eq:main_thm_list2}
\Pr_{\bm{S}\sim \cD^{n}}[\bm{\cA}(\bm{S}) \in \{h_1, \ldots, h_L\}] \geq 1 -\delta, 
\end{equation}
where $\delta \coloneqq n^{-10}$. Since the learning rule can always ignore the extra examples in a sample, we may assume that $n$ is arbitrarily large. In particular, we assume $n>\frac{1}{\epsilon}$. Furthermore, by~\Cref{rem:symmetry}, we assume that $\cA(S)$ and $\cA(S')$ are identically distributed if $S'$ is a reordering of the same examples as in $S$. 

Let $N$ be a large integer that will be determined later. Since $\ldim(\cH)=\infty$, by \Cref{prop:ldim_matrix}, the threshold dimension of $\cH$ is infinite. Thus, since the threshold dimension is at least $N+1$, we may assume (by renaming elements if needed) that $\set{1,\ldots, N} \subseteq X$ and that there exist classifiers $h_1,\ldots,h_{N+1} \in  \mathcal{H}$ with 
\[ h_t(i) = 1  \iff i \ge  t \qquad \text{ for all } i \in [N] \text{ and } t \in [N+1]. \] 
For the remainder of this proof, we focus on the elements in $[N] \subseteq X$ and the classifiers $h_1,\ldots,h_{N+1} \in  \mathcal{H}$, disregarding the others. 

Consider a set $R=\set{x_1,\ldots,x_n} \subseteq [N]$ with $x_1<x_2<\ldots<x_n$. For every $x \in [N]$, define  $\ord_R(x) \in [n+1]$ as
\[\ord_R(x)   \coloneqq  1+|\set{x_i \in R: x_i \leq x}|,\] 
which corresponds to the position of $x$ if it were inserted in the increasing sequence $(x_1,\ldots,x_n)$.   

For each $t \in [n+1]$, consider the output of the learning rule $\cA$ on the labelling of $R$ according to the threshold $t$:
\[\bm{h}^R_{t} \coloneqq \cA(\set{(x_1,0),\ldots,(x_{t-1},0),(x_{t},1),\ldots,(x_n,1)}). \] 

\begin{claim}
\label{claim:SetX}
Let $M$ be a positive integer. Provided that $N$ is sufficiently large, there exists a set $X' \subseteq [N]$ of size $M$, and real numbers $p_{t,k} \in [0,1]$ for $t,k \in [n+1]$ such that the following holds. 

For every subset 
$T=\set{x_1,\ldots,x_n} \subseteq X'$ and every $x \in X' \setminus T$,  we have   
\[p_{t,k}-\delta \le  \Pr[\bm{h}^T_{t}(x)=1]  \le p_{t,k}, \ \ \text{ where }k \coloneqq \ord_T(x)\]
for all $t \in [n+1]$. 
\end{claim}
\begin{proof}
The claim is a consequence of the hypergraph Ramsey theorem. Given any subset $T = \set{x_1, \dots, x_{n+1}} \subseteq [N]$ and $t,k \in [n+1]$, let
\[ q^{T}_{t,k} \coloneqq \Pr[\bm{h}^{T \setminus \set{x_k}}_{t}(x_k)=1], \] 
and let $p^T_{t,k}$ be $q^{T}_{t,k}$ rounded up to an integer multiple of $\delta$, namely
\[ p^T_{t,k} \coloneqq \left\lceil \frac{q^{T}_{t,k}}{\delta} \right\rceil  \delta.\]

Define the ``colour'' of the set $T$ as the matrix
\begin{align*}
    c(T) \coloneqq  [p^T_{t,k}]_{t,k \in [n+1]},
\end{align*}
and note that there are at most $\lceil \frac{2}{\delta} \rceil ^{(n+1) \times (n+1)}$ 
possible colours. By the hypergraph Ramsey theorem~\cite{ramsey1930}, for sufficiently large $N$, there exists $X' \subseteq [N]$ with $|X'| = M$ such that all subsets $T$ of $X'$ of size $n+1$ share the same colour $[p_{t,k}]_{t,k \in [n+1]}$. The set $X'$ and the values $p_{t,k}$ satisfy the claim.
\end{proof}

Let $X^\prime$ be as in \Cref{claim:SetX} for a sufficiently large $M$. By renaming the elements if necessary,  without loss of generality, we assume $X'=[M]$. Let $m^* \coloneqq \lfloor M/2 \rfloor$ be the median of $X'$. Let $\cD$ be the uniform probability distribution over the set
\[\supp(\cD) \coloneqq \set{(x, \1_{[x \ge m^*]}) : x \in [M]}. \]
In other words, we sample $\bm{x}$ uniformly at random from $[M]$ and label it according to the hypothesis $\1_{[x \ge m^*]}$.  We will show that the learner $\cA$ cannot satisfy \Cref{eq:main_thm_list1,eq:main_thm_list2}  for this distribution $\cD$, resulting in a contradiction.

Given $S \in ([M] \times \set{0,1})^n$, let $S_X \in [M]^n$ be the sequence obtained by removing the labels from the examples in $S$. Note that if $\bm{S} \sim \cD^n$, then $\bm{S}_X$ is uniformly distributed over $[M]^n$. 
 
Given a sample $S \in ([M] \times \set{0,1})^n$, let $t(S)$ denote the number of examples in $S$ with label $0$. Let $\Pi$ denote the set of $S \in \supp(\cD)^n$  with the following desired well-spread-ness properties:
\begin{enumerate}
\item  $S$ involves $n$ \emph{distinct} elements in $[M]$. In this case, we identify the sequence $S_X$ with the corresponding $n$-element subset of $[M]$. 
\item $t(S) \in [n/4,3n/4]$.
\item For every interval $I \subseteq [M]$ of size $\frac{M}{8}$, we have 
\begin{equation}
\label{eq:Pi_3}
\left||S_X \cap I| - \frac{n}{8} \right|\le \frac{n}{100}. 
\end{equation}
\item Denoting the elements of $S_X$ by $a_1< a_2< \cdots < a_n$, we have $a_1>\frac{M}{2^n}$, $a_n<M- \frac{M}{2^n}$, and $a_{i+1}>a_i+\frac{M}{2^n}$ for all $i=1,\ldots,n-1$. 
\end{enumerate}
By taking $M$ to be sufficiently large as a function of $n$ and applying Chernoff and union bounds, we have 
\[\Pr_{\bm{S} \sim \cD^n}[\bm{S} \in \Pi] \ge 1- 2^{-\Omega(n)}.\]  
Therefore, we may only focus on the uniformly chosen samples from $\Pi$. Note that the uniform distribution over $\Pi$ corresponds to sampling $\bm{S}\sim \cD^n$ conditioned on $\bm{S}\in \Pi$. 

For every $t \in   [\frac{n}{4},\frac{3n}{4}]$, define 
\[\Pi_t \coloneqq \set{S \in \Pi \ | \ t(S)=t}.\] 

Note that for every $a \in [M]$, we have 
\begin{align*} \Pr_{\bm{S} \sim \Pi_t}[a \in \bm{S}_X] &= \frac{\Pr_{\bm{S} \sim \cD^n}[(a \in \bm{S}_X) \wedge (t(\bm{S})=t)  \wedge (\bm{S} \in \Pi)]} {\Pr_{\bm{S} \sim \cD^n}[(t(\bm{S})=t)  \wedge (\bm{S} \in \Pi)]}  \\ 
&\le \frac{\Pr_{\bm{S} \sim \cD^n}[a \in \bm{S}_X]}{\Pr_{\bm{S} \sim \cD^n}[(t(\bm{S})=t)  \wedge (\bm{S} \in \Pi)]} 
=O_{n,\epsilon}\left(\frac{1}{M}\right),
\end{align*}
where $O_{n,\epsilon}(\cdot)$ indicates that the hidden constants in the bound may depend on $n$ and thus also $\epsilon$.
By the above discussion and \Cref{eq:main_thm_list2}, there exists an integer $t_0 \in [n/4, 3n/4]$ such that
\begin{equation}\label{eq:goodtcond1}
\Pr_{\bm{S} \sim \Pi_{t_0}}[\bm{\cA}(\bm{S}) \in \set{h_1,\dots,h_L}  ] \geq 1-\delta-2^{-\Omega(n)},
\end{equation}
and for every $a \in [M]$,
\begin{equation}\label{eq:goodtcond3}
\Pr_{\bm{S} \sim \Pi_{t_0}}[a \in \bm{S}_X] \le O_{n,\epsilon}\left(\frac{1}{M}\right).
\end{equation}
Fix such a $t_0$ for the rest of the proof.

For every $S \in \Pi_{t_0}$, define the function $f_S:[M] \to [0,1]$ as 
\[f_S(x) \coloneqq \Pr[\bm{\cA}(S)(x)=1].\] 
Since $X'=[M]$ satisfies the assertion of  \Cref{claim:SetX}, there exists values $p_k \coloneqq p_{t_0,k} \in [0,1]$ for $k \in [n+1]$ such that the following holds. For all $S \in \Pi_{t_0}$ and every $x \in X' \setminus S_X$, we have
\begin{equation} 
\label{eq:After_Ramsey} 
    f_S(x) = \Pr[\bm{\cA}(S)(x)=1] \in [p_k - \delta, p_k]  \text{ where }  k = \ord_{S_X}(x).
\end{equation}
Since $h_1, \dots, h_L$ all have low population losses, intuitively, for small $x$, $f_S(x)$ should be close to $0$ and for large $x$, $f_S(x)$ should be close to $1$. Thus, we expect to find $a, b$ such that $|p_b - p_a|$ is large, and consequently there must exist $1\leq c < n+1$ such that $|p_{c+1}-p_c|$ is not too small. 

\begin{claim}
\label{claim:2}
There exists $1\leq c < n+1$ such that $|p_{c+1}-p_c|\geq \frac{1}{2n}$.
\end{claim}
\begin{proof}
Let $\widehat{\bm{x}}$ be uniformly sampled from $[M/8]$.  
Since $h_1,\ldots, h_L$ have loss at most $\epsilon$, and since labeling $\widehat{\bm{x}}$ with $1$ is incorrect, we have 
\[ \Pr_{\widehat{\bm{x}} \sim [M/8]}[h_i(\widehat{\bm{x}})=1] \le 8  \Pr_{\bm{x} \sim [M]}[h_i(\bm{x})=1] \le 8 \epsilon \ \ \text{ for all } i=1,\ldots,L.  \] 
Therefore, using  \Cref{eq:goodtcond1}, we have  
\begin{align}
& \Pr_{\substack{\bm{S}\sim \Pi_{t_0}   \\ \widehat{\bm{x}} \sim [M/8]}} [\bm{\cA}(\bm{S})(\widehat{\bm{x}})= 1]
\nonumber \\ & \qquad\qquad \leq 
 \Pr_{\bm{S}\sim \Pi_{t_0}}[\bm{\cA}(\bm{S}) \notin \set{h_1, \dots, h_L}] + \Pr_{\substack{\bm{S}\sim \Pi_{t_0}  \\ \widehat{\bm{x}} \sim [M/8]}} [\bm{\mathcal{A}}(\bm{S})(\widehat{\bm{x}})=1 \ | \ \bm{\cA}(\bm{S}) \in \set{h_1, \dots, h_L}] \nonumber \\ & \qquad\qquad \leq 
\delta + 2^{-\Omega(n)} + 8\epsilon = O(\epsilon).\label{eq:falsepositives} 
\end{align}

Consider $S \in \Pi$. By \Cref{eq:Pi_3}, every $\widehat{x} \in [M/8] \setminus S_X$ satisfies $\ord_{S_X}(\widehat{x}) < n/4$. Therefore, using \Cref{eq:After_Ramsey}, 
\begin{align*} 
\Pr_{\substack{\bm{S}\sim \Pi_{t_0} \\ \widehat{\bm{x}} \sim [M/8]}} [\bm{\cA}(\bm{S})(\widehat{\bm{x}})= 1] &\ge \Pr_{\substack{\bm{S}\sim \Pi_{t_0} \\ \widehat{\bm{x}} \sim [M/8]}} [\bm{\cA}(\bm{S})(\widehat{\bm{x}})= 1\ | \ \bm{\widehat{x}} \not\in \bm{S}_X]- \Pr_{\substack{\bm{S}\sim \Pi_{t_0} \\ \widehat{\bm{x}} \sim [M/8]}}[ \bm{\widehat{x}} \in \bm{S}_X] \\
& \ge  \min_{k \le n/4} p_k-\delta  - \frac{n}{M/8} =  \min_{k \le n/4} p_k -O(\delta). 
\end{align*}
Combining with \Cref{eq:falsepositives}, we get 
\[
\min_{k \le n/4} p_k = O(\epsilon). 
\] 

Using a similar argument, by considering $\widehat{x} \sim \left[\frac{7M}{8} ,M\right]$, we obtain 
\[
\max_{k \ge 3n/4} p_k = 1-O(\epsilon).
\] 
It follows that 
\[
\bigg|\max_{k \ge 3n/4} p_k  - \min_{k \le n/4} p_k\bigg| \geq \frac{1}{2},
\] 
and therefore there exists some $c \in [n]$ such that $|p_{c+1} - p_c| \geq \frac{1}{2n}$.
\end{proof}

Let $c$ be as in \Cref{claim:2}, and suppose that $c \leq t_0$ without loss of generality.  Call a sample $S \in \Pi_{t_0}$ \emph{good}  if 
\[\Pr[\bm{\cA}(S) \in \set{h_1, \dots, h_L}]\geq 1 - \sqrt{\delta}.\] 

By \Cref{eq:goodtcond1}, we have 
\[\delta + 2^{-\Omega(n)} \ge \Pr_{\bm{S}\sim \Pi_{t_0}}[\bm{\mathcal{A}}(\bm{S})\not\in \{h_1,\dots, h_L\}] \ge \Pr_{\bm{S}\sim \Pi_{t_0}} [ \bm{S} \text{ is not good}] \times \sqrt{\delta}. \] 
Consequently, 
\begin{equation}\label{eq:goodsampprob}
\Pr_{\bm{S}\sim \Pi_{t_0}} [ \bm{S} \text{ is good}] \geq 1-\sqrt{\delta}-\frac{2^{-\Omega(n)}}{\sqrt{\delta}}\geq 1-2\sqrt{\delta}.
\end{equation}

Given any good $S \in \Pi_{t_0}$, define 
\[\overline{h_S} = \sum_{i=1}^L \Pr[\bm{\cA}(S) =h_i | \bm{\cA}(S) \in \set{h_1,\ldots,h_L}] \times h_i,  \] 
and note that by the definition of goodness, we have
\[ |\overline{h}_S(x)-f_S(x)| \le \sqrt{\delta} \ \ \text{ for all } x \in [M]. \] 
The function $\overline{h}_S$ is a convex combination of $h_1,\ldots, h_L$ that pointwise $\sqrt{\delta}$-approximates $f_S$.  Let $G$ be a maximal set of functions $X'\rightarrow [0,1]$ such that 
\begin{enumerate}
    \item Every function $g\in G$ is a convex combination of $h_1,\ldots, h_L$, and 
    \item For every pair of distinct functions $g_1,g_2\in G$, there exists $x\in X^\prime$ such that $|g_1(x)-g_2(x)|\geq \delta$.
\end{enumerate}
By the above two conditions and the above discussion, for any good $S$, there exists $g\in G$ with 
\[\norm{f_S-g}_\infty \le \delta+\sqrt{\delta} \leq 2\sqrt{\delta}.\] 
\begin{claim}\label{claim:small_num_functions}
$|G|\leq O(1/\delta)^L$. 
\end{claim}
\begin{proof}
Denote by $V$ the set of linear combinations of $h_1, \ldots, h_L$, and let $B=V\cap [0,1]^{X'}$. For any $\lambda\leq 1$, define $\lambda B\coloneqq \set{\lambda  g \ | \ g\in B }$. Suppose $m=|G|$, and name the functions in $G$ as $g_1,\ldots, g_m$. Define 
\[
B_i= g_i + \frac{\delta}{2}B.
\]
Now note that $B_1, \ldots, B_m$ are disjoint subsets of $(1+\frac{\delta}{2})B$. Thus the volume of $\cup_i B_i$ is bounded by that of $(1+\delta/2)B$, and we get that $m\leq \left(\frac{1+\delta/2}{\delta/2}\right)^L= O(1/\delta)^L$. 
\end{proof}

Given $S \in \Pi_{t_0}$ where the elements of $S_X$ are ordered as $a_1< a_2< \cdots < a_n$, we define the $i$-th interval of $S$, for $i \in [n+1]$, as $(a_{i-1},a_{i})$ when $i>1$, and $(1,a_1)$ when $i=1$. 

\begin{claim}
\label{claim:intersection}
    There exist $g \in G$ and two good samples $S_1$ and $S_2$ such that \[ \norm{f_{S_1}-g}_\infty \le 2 \sqrt{\delta} \ \text{ and } \  \norm{f_{S_2}-g}_\infty \le 2\sqrt{\delta},\]  and moreover, there is an element $x \in [M]$ that belongs to the $c$-th interval of $S_1$ and the $c+1$-th interval of $S_2$. 
\end{claim}
\begin{proof}
Let $A\subseteq [M/2]$ be the set of $a\in [M/2]$ for which there exists a good sample $S\in \Pi_{t_0}$ such that the $c$-th smallest element of $S_X$ equals $a$. For every $a\in A$, let $S^a$ represent an arbitrary choice of such a good sample. Combining \Cref{eq:goodtcond3,eq:goodsampprob}, we have $|A|= \Omega_{n,\epsilon}(M)$. 
 
Given any $g\in G$, let $A_g$ denote the set of all $a\in A$ with $\norm{f_{S^a}-g}_\infty\leq 2\sqrt{\delta}$. Recall that for every good $S$, there exists $g\in G$ with $\norm{f_{S}-g}_\infty\leq 2\sqrt{\delta}$. Therefore, by \Cref{claim:small_num_functions}, there exists a fixed $g^*\in G$ with 
\[|A_{g^*}|\geq |A|\cdot O(\delta)^L= \Omega_{n,\epsilon}(M).\]

Finally, choosing $M$ sufficiently large guarantees that $|A_{g^*}| \times  M/2^{n+1} \gg M$ and thus there exist $a,b\in A_{g^*}$ such that $a>b$ and  
\[ 2\leq |a-b|\leq M/2^{n+1}. 
\]
Recalling that the gap between every two elements of $S^a_X$ and similarly for $S^b_X$ is at least $M/2^n$, the above guarantees the existence of an element $x$ in the intersection of the $(c+1)$-th interval of $S_1\coloneqq S^a$ and the $c$-th interval of $S_2\coloneqq S^b$. 
\end{proof}

Let $S_1$, $S_2$, $g$, and $x$ be as guaranteed by \Cref{claim:intersection}. In this case, 
\[|f_{S_1}(x)-f_{S_2}(x)| \le \norm{f_{S_1}-g}_\infty + \norm{f_{S_2}-g}_\infty \le 4 \sqrt{\delta}.  \] 
Since $x$ belongs to the $c$-th interval of $S_1$ and $c+1$-th interval of $S_2$, by \Cref{eq:After_Ramsey}, we have $|p_c - f_{S_1}(x)| \le \delta$ and $|p_{c+1} - f_{S_2}(x)| \le \delta$. Therefore, 
\[|p_{c}-p_{c+1}| \le 2 \delta+ 4 \sqrt{\delta}.\] 
For sufficiently large $n$, this inequality contradicts our choice of $c$ from \Cref{claim:2} which satisfies $|p_{c+1}-p_c|\geq \frac{1}{2n}$. 
\end{proof}

\section{\Cref{thm:main}~(ii): Stability from finite Littlestone dimension} \label{sec:(ii)}

In~\cite{BLM20}, Bun, Livni, and Moran showed that every class with finite Littlestone dimension has a \emph{globally} stable learner in the \emph{realizable case}. 

\begin{theorem}[Global Stable Learning from Finite Littlestone Dimension, \cite{BLM20}] \label{thm:global_stable}
Suppose $\cH \subseteq \set{0,1}^X$ satisfy $\ldim(\cH) \leq d$. Then there exists a sample complexity $n:(0,1) \to  \mathbb{N}$ and a learning rule $\bm{\cA}$ such that for every $\epsilon > 0$, and every realizable $\cD$, there exists a hypothesis $h \in \cH$ with  
    \[
      \cL_{\cD}(h) \leq \epsilon\ \ \text{ and }\ \ \Pr_{\bm{S} \sim \cD^n } [\bm{\cA}(\bm{S}) = h] \geq \frac{1}{(d+1)2^{2^d + 1}} \ \text{ where } n=n(\epsilon).
    \]
\end{theorem}

We show that the algorithm of~\cite{BLM20} is already essentially an agnostic $\rho$-global stable learner for these classes for some $\rho:(0,1)\rightarrow (0,1)$. For a minor technical reason,  in the following lemma, we require that the population loss of the distribution $\cD$ is bounded away from $1$. The constant $2/3$ in the statement of the lemma is quite arbitrary and can be replaced by any larger constant strictly less than $1$.

\begin{lemma}[Agnostic $\rho$-Global Stable Learning of Classes with Finite Littlestone Dimension]
\label{lem:For_ii}
Suppose $\cH \subseteq \set{0,1}^X$ satisfy $\ldim(\cH) \leq d$.
There exists a learning rule $\bm{\cA}$, a stability parameter $\rho: (0,1) \to (0,1)$, a sample complexity $n:(0,1) \to \mathbb{N}$,  such that for every $\epsilon > 0$ and every distribution $\cD$ with $\cL_\cD(\cH)< \frac{2}{3}$, there exists a hypothesis $h \in \cH$ with 
\[\cL_{\cD}(h) \leq \cL_{\cD}(\cH) + \eps \ \  \text{and }\ \ 
        \Pr_{\bm{S} \sim \cD^n } [\bm{\cA}(\bm{S}) = h] \geq \rho(\eps).
    \]
Moreover, $\rho$ is given explicitly by $ \rho(\epsilon) = \frac{1}{(d+1)2^{2^d+1}4^{n(\eps)}}.$
   
\end{lemma}
\begin{proof}
Let $\bm{\cA}$ be a globally stable learner for $\cH$ in the realizable case as guaranteed by \Cref{thm:global_stable}, and let $n \coloneqq n_{\ref{thm:global_stable}}(\epsilon/2)$ where $n_{\ref{thm:global_stable}}(\cdot)$ is the sample complexity in \Cref{thm:global_stable}. Note that $\bm{\cA}$ is intended for realizable distributions, but our samples come from a non-realizable distribution $\cD$.

Fix a hypothesis $h^*\in \cH$ with $\gamma \coloneqq \cL_{\cD}(h^*) \leq \cL_{\cD}(\cH)+\frac{\eps}{2} \le \frac{3}{4}$.  Note 
\[\Pr_{(\bm{x},\bm{y}) \sim \cD}[\bm{y}=h^*(\bm{x})]=1-\gamma,\] 
and let $\cD'$ be the distribution obtained by conditioning $\cD$ on the event that the example is consistent with $h^*$.  We have $\cL_{\cD'}(h^*)=0$ and thus $\cD'$ is realizable by $\mathcal{H}$. By our choice of $n$ and \Cref{thm:global_stable}, there is a hypothesis $h\in \cH$ with   
    \[
    \cL_{\cD'}(h)\leq \frac{\eps}{2} \ \ \text{ and } \ \   \Pr_{\bm{S} \sim (\cD')^n}[\bm{\cA}(\bm{S}) = h] \geq \frac{1}{(d+1)2^{2^d + 1}}.
    \]
Since $\gamma \le \frac{3}{4}$, we have 
\[
        \Pr_{\bm{S} \sim \cD^n}[\bm{\cA}(\bm{S}) = h] \ge (1-\gamma)^n \Pr_{\bm{S} \sim (\cD')^n}[\bm{\cA}(\bm{S}) = h]\geq \frac{(1-\gamma)^n}{(d+1)2^{2^d + 1}}\geq \frac{1}{(d+1)2^{2^d+1}4^{n}}.
\]
Moreover, 
\begin{align*}
\cL_{\cD}(h)&=\Pr_{(\bm{x},\bm{y})\sim \cD}[h(\bm{x})\neq \bm{y}] 
\\ &\leq 
\Pr_{(\bm{x},\bm{y})\sim \cD}[h(\bm{x})\neq \bm{y}\ | \ h^*(\bm{x})=\bm{y}]+\Pr_{(\bm{x},\bm{y})\sim \cD}[h^*(\bm{x})\neq \bm{y}] 
\\ &=
\cL_{\cD'}(h)+ \cL_{\cD}(h^*) \leq 
\frac{\eps}{2}+ \cL_{\cD}(\cH) + \frac{\eps}{2}\le \cL_\cD(\cH)+\epsilon .\qedhere
\end{align*}  
\end{proof}

\begin{proof}[Proof of \Cref{thm:main}~(ii)]
Consider the following learning rule: with equal probability $\frac{1}{3}$,  output one of the following hypotheses.
\begin{itemize}
\item The hypothesis that always predicts label $1$ (the all-$1$ hypothesis);
\item The hypothesis that always predicts label $0$ (the all-$0$ hypothesis);
\item The output of the learning rule $\bm{\cA}$ from \Cref{lem:For_ii}.
\end{itemize}
If $\cL_\cD(\cH) \ge \frac{2}{3}$, then we have global stability, since at least one of the all-$1$ or the all-$0$ hypothesis has population loss at most $\frac{1}{2}$ which is smaller than $\cL_\cD(\cH)$. Otherwise, the guarantee of \Cref{lem:For_ii} ensures stability.  
\end{proof}

\bibliographystyle{amsalpha}  
\bibliography{ref} 

\end{document}